%% file: main.tex
\documentclass[letterpaper,10pt,journal,twoside,romanappendices]{IEEEtran}

\input{preambles}
\input{notation-shortcuts}

\begin{document}

\title{Direct Informed Sampling on Riemannian\\ Manifolds via Loewner Order Lower Bounds}

\author{Phone Thiha Kyaw and Jonathan Kelly%
\thanks{Manuscript received: June 1, 2026; Revised August 29, 2026; Accepted September 28, 2026.}%
\thanks{This article was recommended for publication by Associate Editor I. Becerra and Editor O. Stasse upon evaluation of the reviewers' comments.} %
\thanks{The authors are with the Space \& Terrestrial Autonomous Robotic Systems (STARS) Laboratory, University of Toronto Institute for Aerospace Studies (UTIAS), Toronto, ON M3H 5T6, Canada (e-mail: \{phone.thiha, jonathan.kelly\}@robotics.utias.utoronto.ca).}%
\thanks{Digital Object Identifier (DOI): \href{https://doi.org/10.1109/LRA.2026.3746117}{10.1109/LRA.2026.3746117}.}%
}

\markboth{IEEE Robotics and Automation Letters. Preprint Version. Accepted September, 2026}
{Kyaw and Kelly: Direct Informed Sampling on Riemannian Manifolds via Loewner Order Lower Bounds}

\maketitle

\begin{abstract}
Informed sampling techniques accelerate sampling-based motion planners by focusing the search on promising regions of the state space, yet most existing methods rely on Euclidean heuristics that can become inadmissible under configuration-dependent Riemannian metrics.
While scalar eigenvalue bounds restore admissibility by uniformly scaling the Euclidean distance, they discard the directional structure of the metric, producing overly conservative informed sets.
We propose a matrix-valued admissible heuristic that exploits the Loewner order on symmetric positive definite matrices to compute the tightest constant lower bound on the metric tensor %
while retaining directional information discarded by scalar bounds.
The Cholesky factorization of this constant bound maps the lower-bound informed set to a standard prolate hyperspheroid in Euclidean space, enabling direct, rejection-free sampling using existing algorithms.
We then present an incremental algorithm for computing such a constant lower bound on the metric tensor from a sequence of pointwise evaluations.
Experiments on manipulation tasks with a 6-DoF UR5, 7-DoF Franka, and 14-DoF PR2 under three distinct Riemannian metrics show that our heuristic produces consistently tighter informed sets than
those produced by the Euclidean and scalar-eigenvalue heuristics, accelerating convergence across multiple state-of-the-art asymptotically optimal planners.
Our implementation is available in the open-source geodex framework at \url{https://github.com/utiasSTARS/geodex}.
\end{abstract}


\input{sections/introduction}
\input{sections/related-work}
\input{sections/preliminaries}
\input{sections/heuristics}
\input{sections/informed-sets}
\input{sections/experiments}
\input{sections/conclusion}
\input{sections/appendix}

\vspace{-1mm}
\bibliographystyle{IEEEtran}
\bibliography{IEEEabrv,references-short}

\end{document}

%% file: preambles.tex
\usepackage{amsmath,amssymb,amsfonts,amstext}
\usepackage{bm}
\usepackage{graphicx}
\usepackage[ruled,linesnumbered,noend]{algorithm2e}
\usepackage{array}
\usepackage{multirow}
\usepackage{url}
\usepackage{hyperref}
\hypersetup{pdftitle={Direct Informed Sampling on Riemannian Manifolds via Loewner Order Lower Bounds}, pdfauthor={Phone Thiha Kyaw and Jonathan Kelly}}

\usepackage{caption}
\usepackage{subcaption}
\usepackage[table,xcdraw]{xcolor}

\usepackage{siunitx}
\usepackage{microtype}

\usepackage{amsthm}
\newtheorem{theorem}{Theorem}

\newtheorem{corollary}{Corollary}

\newtheorem{definition}{Definition}
\newtheorem{remark}{Remark}

\makeatletter
\patchcmd{\@algocf@start}{-1.5em}{0pt}{}{}
\makeatother

\SetArgSty{textup}

\microtypesetup{expansion=alltext,stretch=20,shrink=40,step=2,factor=1400}

\makeatletter
\newcommand\@IEEEloosebibitem[2][]{%
  \def\@tempa{#1}%
  \ifx\@tempa\@empty \@IEEElatexbibitem{#2}\else \@IEEElatexbibitem[#1]{#2}\fi
  \looseness=-1\relax}
\AtBeginDocument{%
  \patchcmd{\thebibliography}%
    {\def\bibitem{\@IEEEbibitemprefix\@IEEElatexbibitem}}%
    {\def\bibitem{\@IEEEbibitemprefix\@IEEEloosebibitem}}%
    {\typeout{HACK bibloose ok}}{\typeout{HACK bibloose FAIL}}}
\makeatother

\SetAlFnt{\fontsize{9}{10.5}\selectfont}
\SetAlCapFnt{\fontsize{9}{10.5}\selectfont}
\SetAlCapNameFnt{\fontsize{9}{10.5}\selectfont}
\SetAlgoSkip{}

%% file: notation-shortcuts.tex
\usepackage{xparse}

\let\originalleft\left
\let\originalright\right
\renewcommand{\left}{\mathopen{}\mathclose\bgroup\originalleft}
\renewcommand{\right}{\aftergroup\egroup\originalright}

\NewDocumentCommand\Set{m}{ \left\{#1\right\} }
\NewDocumentCommand\Real{}{ \mathbb{R} }

\NewDocumentCommand\Sym{}{ \mathbb{S} }

\NewDocumentCommand\PDMatrices{m}{\Sym^{#1}_{++}}

\NewDocumentCommand\bbm{}{ \begin{bmatrix} } %
\NewDocumentCommand\ebm{}{ \end{bmatrix} }   %
\NewDocumentCommand\T{}{\mathsf{T}}          %

\NewDocumentCommand\Matrix{m}{ \bm{\mathbf{#1}} }

\NewDocumentCommand\Transpose{m}{ \left.{#1}\right.^\T }

\NewDocumentCommand\Inv{m}{{#1}^{-1}}

\NewDocumentCommand\Determinant{m}{ \mathrm{det}\left(#1\right) }
\NewDocumentCommand\Norm{m}{ \left\Vert#1\right\Vert }

\NewDocumentCommand\Identity{}{ \Matrix{I} }

\NewDocumentCommand\ArgMin{m}{ \operatorname*{argmin}_{#1} }

\NewDocumentCommand\LieGroupSO{m}{ \mathrm{SO}(#1) }

\newcommand{\lmin}{\lambda_{\mathrm{min}}}        %
\newcommand{\Glower}{\Matrix{G}_{\mathrm{lower}}} %

\newcommand{\qstart}{q_\mathrm{start}}
\newcommand{\qgoal}{q_\mathrm{goal}}
\newcommand{\xstart}{x_\mathrm{start}}
\newcommand{\xgoal}{x_\mathrm{goal}}

%% file: sections/introduction.tex
\section{Introduction}
\label{sec:introduction}

\looseness=-1
\IEEEPARstart{A}{symptotically} optimal sampling-based planners, such as Rapidly-exploring Random Trees (RRT*) and Batch Informed Trees (BIT*), incrementally explore the configuration space using random sampling and converge to the global optimum as the number of samples grows~\cite{lavalle2001randomized, karaman2011sampling, gammell2020batch}.
Direct informed sampling accelerates convergence by focusing on the \emph{informed set}, which contains configurations that could improve the current solution~\cite{gammell2018informed}.
Under Euclidean path-length costs, the informed set is a prolate hyperspheroid with the start and goal as its foci, and uniform samples can be drawn directly via an affine transformation of an $n$-dimensional unit ball.

\begin{figure}[t]
\centering
\includegraphics[width=0.99\linewidth]{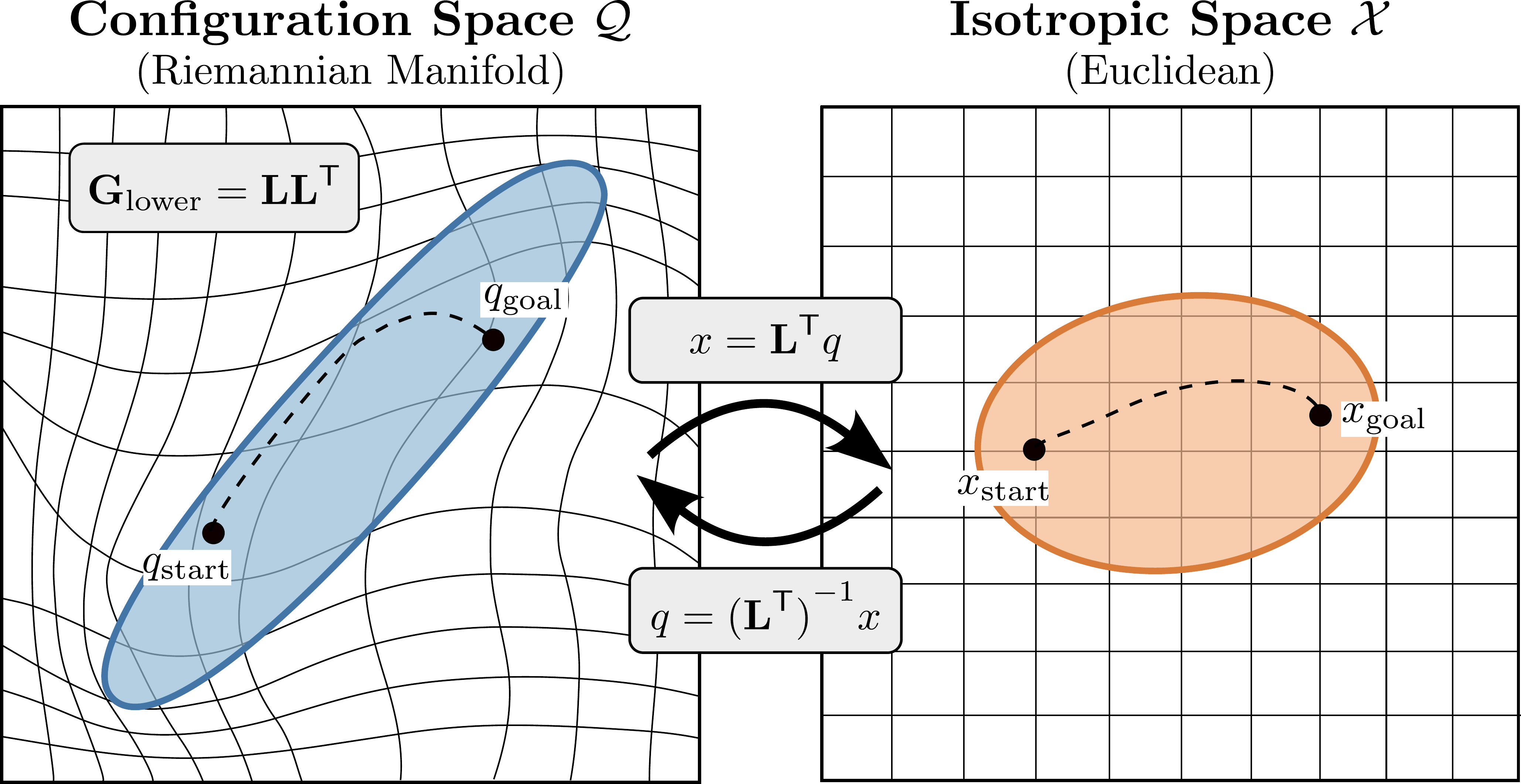}
\caption{
\looseness=-1
Illustration of the proposed direct informed sampling approach for a Riemannian configuration space.
The Cholesky factorization of a Loewner lower bound $\Glower = \Matrix{L}\Matrix{L}^\T$ of the configuration-dependent metric tensor defines a linear isometry $x = \Matrix{L}^\T q$ that maps the anisotropic Riemannian informed set (blue) in the original configuration space to a prolate hyperspheroid (orange) in an isotropic Euclidean space.
Direct sampling from the prolate hyperspheroid in the isotropic space produces rejection-free samples on the Riemannian informed set without requiring knowledge of the full metric field.
}
\vspace{-4mm}
\label{fig:cover}
\end{figure}

\looseness=-1
Many robotics applications, however, require planning with respect to \emph{configuration-dependent} cost structures that go beyond Euclidean distance.
For instance, kinetic energy metrics weight the cost of motion by the manipulator's mass matrix, producing shorter paths through low-inertia configurations~\cite{bullo2019geometric,jaquier2022riemannian,li2024riemannian,kyaw2026geometry}.
Other position-dependent metrics include manipulability~\cite{jaquier2021geometry}, joint stiffness~\cite{saveriano2023learning}, and distance to kinematic singularities~\cite{maric2021riemannian}.
In all such cases, the cost of traversing a path depends on \emph{where} the robot is in configuration space, captured mathematically by a Riemannian metric tensor that varies smoothly with the configuration.

\looseness=-1
Applying informed sampling methods to these costs requires an admissible heuristic---a lower bound on the true cost-to-go---that accounts for the Riemannian metric.
Without one, planners must either sample uniformly over the entire configuration space, forgoing the convergence benefits of informed sampling, or default to the Euclidean distance as the heuristic.
The latter option is not merely conservative; it is inadmissible under many non-Euclidean metrics.
Under the kinetic energy metric, for example, the Euclidean heuristic overestimates the true
distance for 99\% of configuration pairs on
manipulators (as we show in Section~\ref{sec:exp-heuristic-quality}), producing informed sets that fail to contain the true promising region.
Even when a valid scalar heuristic can be constructed, for example, by scaling the Euclidean distance by the minimum eigenvalue of the metric tensor~\cite{peyre2009geodesic}, the resulting informed set is still a prolate hyperspheroid and thus ignores the directional structure of the metric, assigning equal sampling density to all directions regardless of their associated eigenvalues.

\looseness=-1
In this paper, we address these limitations by introducing a direct informed sampling method based on the Loewner order on symmetric positive definite (SPD) matrices.
Our key idea is to use a constant SPD matrix that lower-bounds the configuration-dependent metric tensor everywhere in the space, preserving directional information that scalar bounds discard.
The Cholesky factorization of this lower bound matrix defines an \emph{isotropic} Euclidean space in which the Riemannian informed set reduces to a standard ellipsoid, enabling rejection-free direct sampling via well-known Euclidean techniques~(Figure~\ref{fig:cover}).
To the best of our knowledge, this is the first method for direct informed sampling under configuration-dependent Riemannian metrics.

\looseness=-1
Our main contributions are as follows.
\begin{itemize}
\item We introduce a matrix-valued admissible heuristic via the Loewner order on SPD matrices that is provably at least as tight as the scalar (minimum eigenvalue) bound.
\item We provide an efficient incremental algorithm for computing a maximal constant lower bound on the metric tensor from a sequence of pointwise evaluations.
\item We show that the Loewner lower bound induces an isotropic space in which the Riemannian informed set is a standard prolate hyperspheroid, enabling direct informed sampling using existing Euclidean techniques.
\end{itemize}

%% file: sections/related-work.tex
\section{Related Work}
\label{sec:related-work}

\looseness=-1
Informed sampling has driven much of the recent progress in optimal planning, but existing heuristics for non-Euclidean cost structures are either restricted to kinodynamic settings or do not scale to high dimensions (see Sections~\ref{sec:informed-sampling} and ~\ref{sec:admissible-heuristics-review}).
We construct a matrix-valued admissible heuristic from the Loewner order; Section~\ref{sec:loewner-robotics} reviews its prior uses in robotics.

\vspace{-1mm}
\subsection{Informed Sampling for Optimal Planning}
\label{sec:informed-sampling}

\looseness=-1
Direct informed sampling accelerates convergence of asymptotically optimal planners by restricting sampling to the \emph{informed set}, the region of configuration space that could contain paths shorter than the current best solution~\cite{gammell2018informed}.
Under Euclidean path length, this set is a prolate hyperspheroid, from which uniform samples can be drawn directly without rejection.
Batch Informed Trees (BIT*) generalizes informed sampling to batches of random samples, processing them in order of admissible lower-bound estimates of potential solution cost~\cite{gammell2020batch}.
Adaptively Informed Trees (AIT*) and Effort Informed Trees (EIT*) further relax the need for closed-form admissible heuristics by computing problem-specific cost-to-go estimates through a reverse search from the goal~\cite{strub2022adaptively}.
Greedy-RRT* (G-RRT*) instead accelerates convergence by biasing samples toward the current best path through a greedy variant of the informed set~\cite{kyaw2024greedy}.
Despite these advances, the direct sampling step in all of these methods fundamentally relies on the Euclidean geometry of the informed set.

\looseness=-1
When the cost function is non-Euclidean, the informed set loses its ellipsoidal structure and direct sampling is no longer straightforward~\cite{sehn2024off}.
Prior efforts to address this have focused solely on kinodynamic planning.
Yi et al.\ recast informed set sampling as a sub-level set problem, applying MCMC to draw samples from non-ellipsoidal regions~\cite{yi2018generalizing}.
Kunz et al.\ propose \emph{Hierarchical Rejection Sampling}, which constructs a hierarchy of axis-aligned bounding boxes to improve acceptance rates over naive rejection sampling~\cite{kunz2016hierarchical}.
While both methods handle general informed set geometries, they rely on iterative or rejection-based procedures; extending direct informed sampling to non-Euclidean cost structures remains an open problem~\cite{gammell2021asymptotically}.
Our work addresses this in the geometric planning setting, where the non-Euclidean cost arises from a configuration-dependent Riemannian metric tensor rather than from differential constraints.

\vspace{-1mm}
\subsection{Admissible Heuristics on Non-Euclidean Spaces}
\label{sec:admissible-heuristics-review}

\looseness=-1
Admissible heuristics for non-Euclidean costs have been studied primarily in kinodynamic settings.
One approach uses sum-of-squares programming to derive admissibility certificates from system dynamics~\cite{paden2017verification}.
More recent work extends heuristic-guided search to kinodynamic planning without requiring a closed-form steering function~\cite{przybylski2024asymptotically}.
Other approaches construct admissible bounds from precomputed motion primitive databases~\cite{sakcak2019admissible} or through dimensionality relaxation~\cite{liu2018search}.
All of these methods, however, are tailored to kinodynamic settings and do not address planning under general configuration-dependent Riemannian metrics.

\looseness=-1
A complementary line of work constructs heuristics by approximating geodesic distances rather than computing them exactly~\cite{surazhsky2005fast}.
On surfaces and graphs, landmark-based methods precompute distances to selected reference points and exploit the triangle inequality to construct lower bounds~\cite{peyre2006landmark,paden2017landmark}.
Scaling these methods to general high-dimensional Riemannian manifolds, however, requires solving the Eikonal equation~\cite{peyre2009geodesic,mirebeau2019riemannian}, which becomes computationally intractable as dimension grows.
Several recent works address planning in non-Euclidean configuration spaces without fully resolving the heuristic problem.
One line extends Graphs of Convex Sets to Riemannian manifolds via convex relaxation~\cite{cohn2025non}, and another establishes asymptotic optimality of RRT* and PRM* under general non-Euclidean metrics~\cite{lukyanenko2023probabilistic}, but neither provides explicit heuristics for informed sampling.
Recently, Kyaw and Kelly proposed a geometry-aware planning framework using a midpoint-based geodesic approximation that is accurate but not provably admissible, explicitly identifying heuristic design as important future work~\cite{kyaw2026geometry}.

\vspace{-1mm}
\subsection{Loewner Order in Robotics}
\label{sec:loewner-robotics}

\looseness=-1
The Loewner order on symmetric matrices is foundational in robust control, where linear matrix inequalities encode stability and robustness constraints~\cite{boyd1994linear}.
The closest planning analogues appear in belief-space planning, where scalar bounds on the maximum eigenvalue of the covariance matrix make sampling-based search tractable~\cite{bopardikar2016robust,shan2017belief}, but these bounds collapse a configuration-dependent matrix to a scalar, discarding all directional information.
Contraction-based approaches retain more structure; Singh et al.\ optimize a matrix-valued contraction metric via sum-of-squares programming and use a constant Loewner lower bound to define ellipsoidal robustness tubes around tracked trajectories~\cite{singh2017robust}.
In trajectory optimization, CHOMP and STOMP precompute a fixed precision matrix from finite-difference operators to define a Riemannian metric on the trajectory space~\cite{kalakrishnan2011stomp,zucker2013chomp}; however, this matrix encodes temporal smoothness rather than configuration-space geometry.
More broadly, semidefinite programming, which optimizes over the cone ordered precisely by the Loewner order, has been applied to related robotic problems such as collision-free region construction~\cite{deits2015computing} and state estimation~\cite{holmes2024semidefinite}.
In contrast, we exploit the Loewner order to construct a \emph{matrix-valued} admissible heuristic that preserves the full directional structure of the metric tensor while enabling direct ellipsoidal informed sampling.

%% file: sections/preliminaries.tex
\vspace{-1mm}
\section{Preliminaries}
\label{sec:preliminaries}

This section reviews the key concepts and the notation used throughout the paper.

\vspace{-1mm}
\subsection{Riemannian Configuration Spaces}
\label{sec:riemannian-configuration-spaces}

\looseness=-1
Let $\mathcal{Q} \subseteq \Real^n$ denote the configuration space manifold, with each $q \in \mathcal{Q}$ representing a configuration of the robot.
We assume throughout that $\mathcal{Q}$ admits a single global chart, so that a configuration is identified with a point in $\Real^n$ and the difference $q_x - q_y$ is well defined.
We also assume that $\mathcal{Q}$ is compact.
A Riemannian metric on $\mathcal{Q}$ is a smoothly varying positive definite matrix $\Matrix{G}(q) \in \PDMatrices{n}$ that defines an inner product at each configuration~\cite{lee2018introduction}.
The metric induces a notion of path length: for a piecewise smooth curve $\pi : [0,1] \to \mathcal{Q}$, the Riemannian arc length is
\begin{equation}
\label{eqn:arc-length}
L(\pi) = \int_0^1 \sqrt{\Transpose{\dot{\pi}(t)}\, \Matrix{G}(\pi(t))\, \dot{\pi}(t)} \; dt.
\end{equation}
The Riemannian distance between two configurations is defined as the infimum of arc lengths over all connecting curves,
\begin{equation}
\label{eqn:distance}
d(q_x, q_y) = \inf_{\pi}\, L(\pi), \quad \pi(0) = q_x,\; \pi(1) = q_y.
\end{equation}
When $\Matrix{G}(q) = \Identity$ for all $q$, the Riemannian distance reduces to the standard Euclidean distance $\Norm{q_x - q_y}_2$.

\vspace{-2mm}
\subsection{Informed Sampling for Optimal Planning}
\label{sec:informed-sampling-for-optimal-planning}

\looseness=-1
Given a start configuration $\qstart$ and a goal configuration $\qgoal$, the optimal motion planning problem seeks a collision-free path $\pi^*$ that minimizes the Riemannian arc length in~\eqref{eqn:arc-length} from $\qstart$ to $\qgoal$.
Asymptotically optimal sampling-based planners, such as RRT*~\cite{karaman2011sampling}, incrementally build a search tree over $\mathcal{Q}$ and converge to $\pi^*$ as the number of samples grows.
This convergence can often be accelerated by using a distance estimate $\hat{d}$ that lower-bounds the true distance (\emph{admissible heuristic}), i.e., $\hat{d}(q_x, q_y) \leq d(q_x, q_y)$ for all $q_x, q_y \in \mathcal{Q}$, enabling the planner to prune configurations that provably cannot improve the current best solution.
Many informed planners rely on such heuristics to focus sampling in promising regions of the configuration space~\cite{gammell2020batch, strub2022adaptively, kyaw2024greedy}.

\looseness=-1
Specifically, let $c_{\mathrm{best}}$ denote the cost of the current best solution.
The \emph{informed set} is the subset of configurations that could lie on a shorter path~\cite{gammell2018informed},
\begin{equation}
\label{eqn:informed-set-prelim}
X_{\hat{f}} = \Set{q \in \mathcal{Q} : \hat{d}(\qstart, q) + \hat{d}(q, \qgoal) < c_{\mathrm{best}}}.
\end{equation}
Under the Euclidean metric ($\Matrix{G} = \Identity$), the natural heuristic is $\hat{d}(q_x, q_y) = \Norm{q_x - q_y}_2$, and $X_{\hat{f}}$ reduces to a prolate hyperspheroid with foci $\qstart$ and $\qgoal$, from which uniform samples can be drawn directly via an affine transformation of the unit ball.
However, when the cost is defined by a configuration-dependent Riemannian metric, the Euclidean distance is not necessarily admissible, and the geometry of $X_{\hat{f}}$ under a valid heuristic is not necessarily a standard hyperspheroid.

\subsection{Loewner Order on Symmetric Matrices}

The \emph{Loewner order} is a partial order on symmetric matrices induced by the
cone of positive semidefinite (PSD) matrices~\cite{bhatia2007positive}.

\begin{definition}[Loewner order]
\label{def:loewner-order}
For symmetric matrices $\Matrix{A}, \Matrix{B} \in \Sym^n$, we write $\Matrix{A} \preceq \Matrix{B}$ if $\Matrix{B} - \Matrix{A}$ is positive semidefinite, i.e., $v^\T \Matrix{A}\, v \leq v^\T \Matrix{B}\, v$ for all $v \in \Real^n$.
\end{definition}

The Loewner order is not a total order; that is, not all symmetric positive definite (SPD) matrices are comparable, since $\Matrix{B} - \Matrix{A}$ may be indefinite.
In fact, a unique \emph{greatest lower bound} (meet) generally does not exist for symmetric matrices unless the matrices are comparable in the Loewner order~\cite{kadison1951order}.
A finite collection $\smash{\Set{\Matrix{A}_1, \ldots, \Matrix{A}_k} \subset \PDMatrices{n}}$ nevertheless always admits SPD \emph{maximal lower bounds}, that is, lower bounds that cannot be increased in the Loewner order while remaining dominated by every $\Matrix{A}_i$.\footnote{The order is defined on all symmetric matrices; here we restrict to Riemannian metric evaluations, which are symmetric positive definite.}
Such maximal lower bounds are not unique in general~\cite{stott2016maximal}.
For such a collection, we compute a lower bound iteratively via pairwise meets, as we describe in
Section~\ref{sec:heuristics}.\footnote{With a slight abuse of terminology, we use the word ``meet'' to refer to this specific maximal lower bound of two SPD matrices.}

%% file: sections/heuristics.tex
\section{Admissible Heuristics via Loewner Bounds}
\label{sec:heuristics}

We first introduce a lower bound that exploits the Loewner order on SPD matrices to produce tighter estimates than scalar eigenvalue bounds.
We then describe an iterative algorithm to compute this bound from a sequence of metric evaluations.

\vspace{-1mm}
\subsection{Matrix Lower Bound Heuristic}
\label{sec:matrix-lower-bound-heuristic}

We begin by formalizing the notion of a constant matrix that lower-bounds the configuration-dependent metric tensor everywhere on the manifold.

\looseness=-1
\begin{definition}[Loewner lower bound]
\label{def:loewner-lower-bound}
A constant symmetric positive definite matrix $\Glower \in \PDMatrices{n}$ is a \emph{Loewner lower bound} for the metric tensor $\Matrix{G}$ if it satisfies the Loewner order (Definition~\ref{def:loewner-order}) everywhere on $\mathcal{Q}$, i.e.,
$\Glower \preceq \Matrix{G}(q), \forall\, q \in \mathcal{Q}$.
\end{definition}

\looseness=-1
A natural approach to satisfy this definition is to scale the Euclidean distance by the global minimum eigenvalue of $\Matrix{G}(q)$, which yields the simplest lower bound estimate for Riemannian geodesic distance~\cite{peyre2009geodesic,singh2017robust}.

\begin{definition}[Scalar distance lower bound]
\label{def:scalar-distance-lower-bound}
Let $\lmin = \min_{q \in \mathcal{Q}} \smash{\lmin\big(\Matrix{G}(q)\big)}$ denote the global minimum eigenvalue of the metric tensor, and define the isotropic lower-bound matrix $\Matrix{G}_{\lambda} = \lmin\Identity$.
The scalar distance lower bound is
\begin{equation}
\label{eqn:scalar-distance-lower-bound}
\hat{d}_{\lambda}(q_x, q_y) = \sqrt{\lmin}\,\Norm{q_x - q_y}_2 = \Norm{q_x - q_y}_{\Matrix{G}_{\lambda}} .
\end{equation}
\end{definition}

\looseness=-1
While the scalar bound ensures admissibility by scaling distances uniformly, it can be overly conservative.
We propose a more general distance estimate that replaces the configuration-dependent Riemannian metric with a %
constant matrix lower bound. %

\looseness=-1
\begin{definition}[Matrix distance lower bound]
\label{def:matrix-distance-lower-bound}
Let $\Glower \in \PDMatrices{n}$ be a Loewner lower bound for $\Matrix{G}$ (Definition~\ref{def:loewner-lower-bound}), and let $\Glower = \Matrix{L}\Matrix{L}^\T$ denote its Cholesky factorization.
The matrix distance lower bound is the function $\hat{d} \colon \mathcal{Q} \times \mathcal{Q} \to \Real_{\geq 0}$,
\begin{equation}
\label{eqn:matrix-distance-lower-bound}
\hat{d}(q_x, q_y) = \Norm{q_x - q_y}_{\Glower} = \Norm{\Matrix{L}^\T (q_x - q_y)}_2.
\end{equation}
\end{definition}
Note that the scalar bound in \eqref{eqn:scalar-distance-lower-bound} is a special case of the matrix bound, corresponding to the (isotropic) Loewner lower bound $\Glower = \Matrix{G}_{\lambda} = \lmin \Identity$.
If $\Matrix{G}_{\lambda} \preceq \Glower$, the matrix bound is at least as tight as the scalar bound, and strictly tighter for displacements $v$ satisfying $\Transpose{v}(\Glower-\Matrix{G}_{\lambda})v>0$.

\looseness=-1
\begin{theorem}[Tightness relative to scalar bound]
If $\Matrix{G}_{\lambda} \preceq \Glower$, then %
the matrix distance lower bound $\hat{d}$ dominates the scalar bound $\hat{d}_{\lambda}$
\begin{equation*}
\hat{d}_{\lambda}(q_x, q_y) \leq \hat{d}(q_x, q_y), \quad \forall\, q_x, q_y \in \mathcal{Q},
\end{equation*}
with equality if and only if either $q_x = q_y$, or $q_x - q_y$ is an eigenvector of $\Glower$ corresponding to eigenvalue $\lmin$.
\end{theorem}

\begin{proof}
Let $v = q_x - q_y$.
Since $\lmin \Identity \preceq \Glower$,
\begin{equation*}
\Transpose{v}\, \Glower\, v \geq \Transpose{v}\, (\lmin \Identity)\, v = \lmin \Norm{v}_2^2.
\end{equation*}
Taking square roots gives $\hat{d}(q_x, q_y) \geq \hat{d}_{\lambda}(q_x, q_y)$.
Equality holds if and only if $\Glower\, v = \lmin\, v$, i.e., $v = 0$, or $v$ is an eigenvector of $\Glower$ corresponding to $\lmin$.
\end{proof}

As the following theorem shows, the Loewner bound on the metric tensor is sufficient to guarantee that $\hat{d}$ in \eqref{eqn:matrix-distance-lower-bound} never overestimates the true geodesic distance.

\looseness=-1
\begin{theorem}[Admissibility of matrix distance lower bound]
\label{thm:admissibility}
If $\Glower \preceq \Matrix{G}(q)$ for all $q \in \mathcal{Q}$, then the matrix distance lower bound $\hat{d}$ is admissible for the true geodesic distance $d$,
\begin{equation*}
\hat{d}(q_x, q_y) \leq d(q_x, q_y), \quad \forall\, q_x, q_y \in \mathcal{Q}.
\end{equation*}
\end{theorem}

\begin{proof}
Let $\pi : [0,1] \to \mathcal{Q}$ be any piecewise smooth curve from $q_x$ to $q_y$.
Since $\Glower \preceq \Matrix{G}(q)$ for all $q \in \mathcal{Q}$,
\begin{align*}
L(\pi) &= \int_0^1 \sqrt{\Transpose{\dot{\pi}}\, \Matrix{G}(\pi)\, \dot{\pi}} \; dt \geq \int_0^1 \Norm{\Matrix{L}^\T \dot{\pi}}_2 \; dt \\[3pt]
&\geq \Norm{\Matrix{L}^\T (q_y - q_x)}_2 = \hat{d}(q_x, q_y).
\end{align*}
The infimum over all curves yields $\hat{d}(q_x, q_y)\!\leq\!d(q_x, q_y)$.
\end{proof}

An immediate consequence is that the matrix lower bound generalizes the standard Euclidean distance heuristic.

\looseness=-1
\begin{corollary}[Euclidean special case]
\label{cor:euclidean-recovery}
When $\Glower = \Identity$, the matrix distance lower bound reduces to the standard Euclidean distance,
$\hat{d}(q_x, q_y) = \sqrt{(q_x-q_y)^\T \Identity (q_x-q_y)} = \Norm{q_x - q_y}_2$.
\end{corollary}

\looseness=-1
While Corollary~\ref{cor:euclidean-recovery} recovers the Euclidean case, the converse does not hold; that is, using the Euclidean distance as a heuristic for any Riemannian metric can easily violate admissibility.

\looseness=-1
\begin{remark}
Corollary~\ref{cor:euclidean-recovery} shows that the Euclidean distance corresponds to $\Glower = \Identity$, which by Theorem~\ref{thm:admissibility} is admissible if $\Identity \preceq \Matrix{G}(q)$ for all $q$, i.e., $\lmin(\Matrix{G}(q)) \geq 1$ everywhere.
When this fails, the metric contracts relative to the Euclidean norm along the corresponding eigendirection, causing the Euclidean distance to overestimate true distances and violate admissibility.
The scalar bound (Definition~\ref{def:scalar-distance-lower-bound}) recovers admissibility by scaling distances uniformly; our matrix bound does so directionally, tightening only in the directions where the metric permits.
\end{remark}

\subsection{Computing the Loewner Lower Bound}
\label{sec:computing-loewner-lower-bound}

\looseness=-1
The matrix heuristic in Definition~\ref{def:loewner-lower-bound} requires a constant $\Glower$ that is dominated by $\Matrix{G}(q)$ at every configuration in $\mathcal{Q}$.
Computing the tightest such bound is in general intractable, as it requires solving a semi-infinite program over the entire configuration space.
We reduce this problem to a sequence of pairwise updates: given the current bound $\Glower$ and a new metric evaluation $\Matrix{G}_{\mathrm{new}}$, Algorithm~\ref{alg:loewner-meet} computes the \emph{Loewner meet}, a maximal SPD matrix dominated by both.

\setlength{\textfloatsep}{2mm}%
\begin{algorithm}[!t]
\caption{Loewner meet ($\Matrix{L}$, $\Matrix{G}_{\mathrm{new}}$)}
\label{alg:loewner-meet}
$\Matrix{S} \gets \Inv{\Matrix{L}}\, \Matrix{G}_{\mathrm{new}}\, \Inv{(\Matrix{L}^\T)}$\label{algo1:line1}\;
$\Matrix{V}, \text{diag}(\lambda_1, \dots, \lambda_n) \gets \text{EigenDecomposition}(\Matrix{S})$\label{algo1:line2}\;
\If{$\min_k \lambda_k < 1$\label{algo1:line3}}{
    $\tilde\lambda_k \gets \min(\lambda_k, 1)$ for each $k \in \{1, \dots, n\}$\label{algo1:line4}\;
    $\tilde{\Matrix{\Lambda}} \gets \text{diag}(\tilde\lambda_1, \dots, \tilde\lambda_n)$\label{algo1:line5}\;
    $\Glower \gets \Matrix{L}\, \Matrix{V}\, \tilde{\Matrix{\Lambda}}\, \Matrix{V}^\T \Matrix{L}^\T$\label{algo1:line6}\;
    $\Matrix{L} \gets \text{Cholesky}(\Glower)$\label{algo1:line7}\;
}
\Return{$\Matrix{L}$}\label{algo1:line8}
\end{algorithm}

\looseness=-1
The algorithm maintains the Cholesky factor $\Matrix{L}$ of the current bound $\Glower$.
It first transforms $\Matrix{G}_{\mathrm{new}}$ into a whitened coordinate system where $\Glower$ is the identity (Line~\ref{algo1:line1}).
The eigenvalues of the transformed matrix $\Matrix{S}$ indicate how $\Matrix{G}_{\mathrm{new}}$ compares to $\Glower$ along each principal direction (Line~\ref{algo1:line2}).
If no eigenvalue falls below $1$, then $\Matrix{G}_{\mathrm{new}}$ already dominates $\Glower$ in every direction and the bound is left unchanged (Line~\ref{algo1:line3}).
Otherwise, eigenvalues exceeding $1$ correspond to directions where $\Matrix{G}_{\mathrm{new}}$ already dominates $\Glower$ and are clamped; eigenvalues less than $1$ indicate directions where $\Matrix{G}_{\mathrm{new}}$ is smaller, requiring the bound to be loosened (Line~\ref{algo1:line4}).
Reconstructing the matrix in the original coordinate system yields a maximal lower bound of $\Glower$ and $\Matrix{G}_{\mathrm{new}}$ (Lines~\ref{algo1:line5}--\ref{algo1:line6}), whose Cholesky factor is updated and returned for the next call (Lines~\ref{algo1:line7}--\ref{algo1:line8}).

\begin{remark}
Algorithm~\ref{alg:loewner-meet} is a generic subroutine: given any sequence of metric evaluations $\Matrix{G}(q_1), \ldots, \Matrix{G}(q_N)$, applying it iteratively from $\Glower = \Matrix{G}(q_1)$ yields a valid Loewner lower bound for all observed matrices.
The only design choice is the selection strategy for the configurations $q_i$, which can range from uniform batch sampling to targeted optimization; we detail the specific strategy used in Section~\ref{sec:experiments}.
\end{remark}

\looseness=-1
\begin{remark}
Admissibility follows from Theorem~\ref{thm:admissibility} whenever $\Glower \preceq \Matrix{G}(q)$ at every $q \in \mathcal{Q}$.
Algorithm~\ref{alg:loewner-meet} enforces this condition for the input configurations only, and the optimization-based selection of Appendix~\ref{sec:appendix-opt-bound} certifies it only at the computed local optima.
The bound returned by either procedure is admissible over the explored configurations rather than provably over all of $\mathcal{Q}$.
Section~\ref{sec:exp-heuristic-quality} verifies the computed bounds empirically, finding no violation over 10,000 configuration pairs in any of the nine scenarios.%
\end{remark}

%% file: sections/informed-sets.tex
\section{Riemannian Informed Sets}
\label{sec:informed-sets}

\begin{figure*}[!t]
\centering
\includegraphics[width=0.97\textwidth]{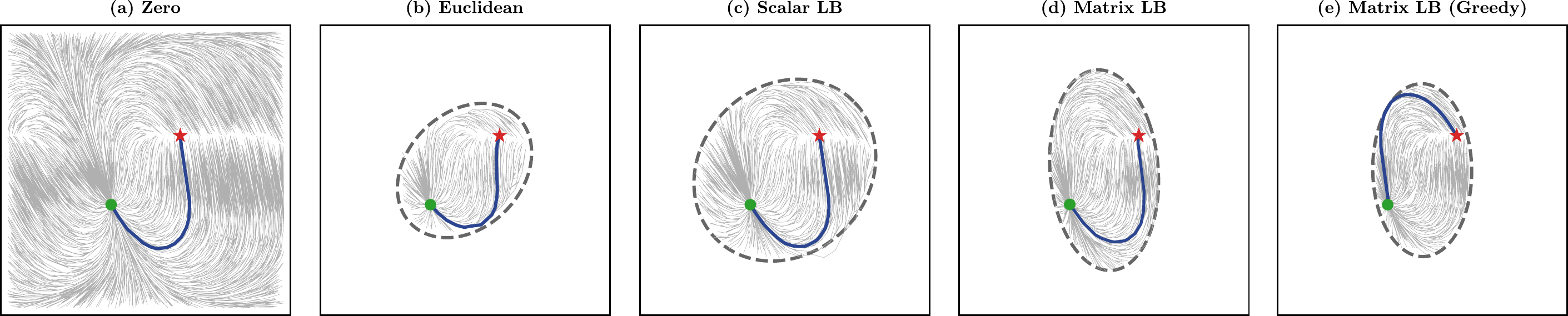}
\caption{%
Search trees and informed sets for a 2-DoF planar manipulator planning from start (\textcolor{green!50!black}{$\bullet$}) to goal (\textcolor{red}{$\star$}) under the kinetic energy Riemannian metric.
(a)~With no heuristic, the planner samples uniformly over $\mathcal{Q}$, producing a dense tree with no directional preference.
(b)~The Euclidean heuristic is inadmissible for this metric; the resulting isotropic prolate hyperspheroid (PHS) misguides sampling and yields a suboptimal solution.
(c)~The scalar eigenvalue bound produces an admissible but isotropic PHS that does not exploit the metric's anisotropy.
(d)~Our matrix-valued lower bound yields an anisotropic ellipsoid aligned with the metric structure, focusing sampling along low-cost directions.
(e)~The greedy variant further concentrates samples near the current best solution.
}
\label{fig:qualitative-2d}
\vspace{-3mm}
\end{figure*}

Having established an admissible distance bound in Section~\ref{sec:heuristics}, we now show that the resulting informed set has a clean geometric characterization.
By exploiting this geometry, we develop a direct, \emph{rejection-free} sampling approach that strictly improves upon scalar-bound heuristics.

\vspace{-1mm}
\subsection{Isometric Embedding}
\label{sec:isometric-embedding}

\looseness=-1
The Cholesky factorization $\Glower = \Matrix{L}\Matrix{L}^\T$ of the lower bound matrix defines a natural coordinate transformation from the configuration space $\mathcal{Q}$ to an isotropic Euclidean space $\mathcal{X} \subseteq \Real^{n}$.
Under this transformation, the anisotropic weighted norm induced by $\Glower$ reduces to the standard Euclidean norm, enabling the use of classical Euclidean geometric tools in $\mathcal{X}$.

\begin{theorem}[Isometric embedding]
\label{thm:isometry}
The map $\phi \colon \mathcal{Q} \to \mathcal{X}$ defined by $\phi(q) = \Matrix{L}^\T q$ is a linear isometry from $(\mathcal{Q}, \Glower)$ to $(\mathcal{X}, \Identity)$.
That is, for all $q_1, q_2 \in \mathcal{Q}$,
\begin{equation*}
\hat{d}(q_1, q_2) = \Norm{\phi(q_1) - \phi(q_2)}_2.
\label{eq:isometry}
\end{equation*}
\end{theorem}

\begin{proof}
By direct computation and using Definition~\ref{def:matrix-distance-lower-bound},
\begin{equation*}
\Norm{\phi(q_1) - \phi(q_2)}_2 = \Norm{\Matrix{L}^\T(q_1 - q_2)}_2 = \hat{d}(q_1, q_2). \qedhere
\end{equation*}
\end{proof}

\looseness=-1
The inverse map $\Inv{\phi}(x) = \Matrix{L}^{-\T} x$ transforms points from $\mathcal{X}$ back to the original configuration space $\mathcal{Q}$.
We now use this isometry to characterize the geometry of the informed set.
For a cost bound $c_{\mathrm{best}}$, specializing~\eqref{eqn:informed-set-prelim} to the matrix-valued heuristic of Definition~\ref{def:matrix-distance-lower-bound} gives the Riemannian informed set
\begin{equation}
\label{eq:informed-set}
X_{\hat{f}}^{\mathcal{R}} = \Set{q \in \mathcal{Q} : \hat{d}(\qstart, q) + \hat{d}(q, \qgoal) < c_{\mathrm{best}}}.
\end{equation}
Here and in the remainder of this section, we take \eqref{eq:informed-set} over all of $\Real^n$, i.e., we disregard the boundary of $\mathcal{Q}$; samples outside $\mathcal{Q}$ are rejected in Algorithm~\ref{alg:direct-sampling}.
By Theorem~\ref{thm:isometry}, the heuristic sum in \eqref{eq:informed-set} equals $\smash{\Norm{\xstart - x}_2 + \Norm{x - \xgoal}_2}$, where $\smash{\xstart = \phi(\qstart)}$, $\smash{\xgoal = \phi(\qgoal)}$, and $x = \phi(q)$.
The set of points satisfying $\Norm{\xstart - x}_2 + \Norm{x - \xgoal}_2 \leq c_{\mathrm{best}}$ is precisely a \emph{prolate hyperspheroid} (PHS) in $\smash{\mathcal{X}}$ with foci $\xstart$ and $\xgoal$, which we denote $\smash{X_{\hat{f}}^{\mathcal{X}}}$~\cite{gammell2018informed}.
Since $\phi$ is an invertible linear map, the preimage of $X_{\hat{f}}^{\mathcal{X}}$ is an ellipsoid in the original configuration space.

\looseness=-1
\begin{corollary}[Riemannian informed set is an ellipsoid]
\label{cor:ellipsoid}
The Riemannian informed set $X_{\hat{f}}^{\mathcal{R}}$ in~\eqref{eq:informed-set} satisfies
\begin{equation}
X_{\hat{f}}^{\mathcal{R}} = \Inv{\phi}\!(X_{\hat{f}}^{\mathcal{X}}).
\label{eq:ellipsoid}
\end{equation}
In the original coordinates, $\smash{X_{\hat{f}}^{\mathcal{R}}}$ under $\Glower = \Matrix{L}\Matrix{L}^\T$ is an ellipsoid whose principal axes are rotated and scaled by $\Matrix{L}^{-\T}$ relative to those of $X_{\hat{f}}^{\mathcal{X}}$.
\end{corollary}

\looseness=-1
Compared to the scalar bound (Definition~\ref{def:scalar-distance-lower-bound}), which produces an isotropic PHS regardless of the metric's anisotropy, the matrix bound produces an ellipsoid that is elongated along low-cost directions and compressed along high-cost directions.
Consequently, this directional structure of $\smash{X_{\hat{f}}^{\mathcal{R}}}$ focuses sampling in regions most likely to improve the current solution (Figure~\ref{fig:qualitative-2d}).

\vspace{-2mm}
\subsection{Direct Sampling}
\label{sec:direct-sampling}

The ellipsoidal characterization in Corollary~\ref{cor:ellipsoid} enables a direct sampling procedure for $\smash{X_{\hat{f}}^{\mathcal{R}}}$.
Instead of drawing samples from the entire $\mathcal{Q}$ and rejecting those outside $\smash{X_{\hat{f}}^{\mathcal{R}}}$, we sample uniformly from $X_{\hat{f}}^{\mathcal{X}}$ in the isotropic space and transform back to $\mathcal{Q}$.
This procedure, presented in Algorithm~\ref{alg:direct-sampling}, produces uniform samples on $\smash{X_{\hat{f}}^{\mathcal{R}}}$ without any geometric rejection.

\setlength{\textfloatsep}{2mm}%
\begin{algorithm}[!t]
\caption{Sampling from $X_{\hat{f}}^{\mathcal{R}}$ ($\Matrix{L}$, $\qstart$, $\qgoal$, $c_{\mathrm{best}}$)}
\label{alg:direct-sampling}
$\xstart \gets \Matrix{L}^\T \qstart$, $\xgoal \gets \Matrix{L}^\T \qgoal$\label{algo2:line1}\;
$x \gets \mathrm{Uniform}(X_{\hat{f}}^{\mathcal{X}})$\label{algo2:line2}\;
$q \gets \Matrix{L}^{-\T} x$\label{algo2:line3}\;
\If{$q \in \mathcal{Q}$\label{algo2:line4}}{
    \Return{$q$}\;
}
\Return{reject and resample}\label{algo2:line5}
\end{algorithm}

\looseness=-1
The algorithm first maps $\qstart$ and $\qgoal$ to $\mathcal{X}$ via $\phi$ (Line~\ref{algo2:line1}).
It then draws a uniform sample from $X_{\hat{f}}^{\mathcal{X}}$ using the standard PHS decomposition of~\cite{gammell2018informed} (Line~\ref{algo2:line2}).
Finally, it maps the sample back to the original configuration space $\mathcal{Q}$ via $\Inv{\phi}$ (Line~\ref{algo2:line3}).
Because $\Inv{\phi}$ is an invertible linear map, the change of variables preserves uniformity with respect to the Lebesgue measure: the resulting sample is uniform on $\smash{X_{\hat{f}}^{\mathcal{R}}} \subset \mathcal{Q}$ in the same sense as in standard Euclidean informed sampling.\footnote{
This is uniformity with respect to the Lebesgue measure on $\mathcal{Q}$, not the true Riemannian volume induced by $\Matrix{G}(q)$; the latter would require knowledge of the full metric field, which our heuristic explicitly avoids.}
The only source of rejection is the boundary of $\mathcal{Q}$, that is, the ellipsoid may extend beyond the configuration space (Line~\ref{algo2:line4}), which is identical to the rejection in standard PHS sampling.
There is no additional geometric rejection from the informed set, unlike heuristic-guided rejection sampling for generic distance functions.

\vspace{-1mm}
\subsection{Geometric Interpretation}
\label{sec:geometric-interpretation}

\looseness=-1
The spectral decomposition of $\Glower$ provides a concrete geometric picture of the Riemannian informed set.
Let $\Glower = \Matrix{V}\Matrix{\Lambda}\Transpose{\Matrix{V}}$ be the eigendecomposition, where each column of $\Matrix{V}$ is an eigenvector with associated eigenvalue $\lambda_i$.
Each eigenvector defines a principal direction of the metric lower bound, and the corresponding eigenvalue encodes the cost of motion along that direction.
The isometry $\phi(q) = \Matrix{L}^\T q$ rotates the configuration space to align with these principal directions and scales each axis by $\sqrt{\lambda_i}$, so that all directions carry equal cost in $\mathcal{X}$.
The inverse map $\Inv{\phi}$ reverses this normalization, warping the isotropic PHS $\smash{X_{\hat{f}}^{\mathcal{X}}}$ into the anisotropic ellipsoid $\smash{X_{\hat{f}}^{\mathcal{R}}}$ in the original coordinates.
Under $\phi^{-1}$, lengths along the $i$th eigenvector of $\Glower$ are scaled by $\smash{1/\sqrt{\lambda_i}}$, that is, the ellipsoid extends further along directions where motion is inexpensive and contracts where it is costly.

\looseness=-1
This ellipsoidal characterization also yields a closed-form expression for the Lebesgue measure (volume) of the Riemannian informed set $\smash{X_{\hat{f}}^{\mathcal{R}}}$;
note that this upper bounds the volume of the true Riemannian informed set defined by the geodesic distance $d$, i.e., $\smash{X_{f}^{\mathcal{R}}} \subseteq \smash{X_{\hat{f}}^{\mathcal{R}}}$.
Let $\mu(\cdot)$ denote the Lebesgue measure in $\Real^n$.
Since $\Inv{\phi}$ is a linear map with Jacobian determinant $1/\sqrt{\Determinant{\Glower}}$, the change-of-variables formula gives
\begin{equation}
\label{eq:ris-volume}
\mu\left(X_{\hat{f}}^{\mathcal{R}}\right) = \frac{\mu_{\mathrm{PHS}}\left(c_{\mathrm{best}},\, d_{\mathrm{foci}}\right)}{\sqrt{\Determinant{\Glower}}},
\end{equation}
where $d_{\mathrm{foci}} = \hat{d}(\qstart, \qgoal)$ is the heuristic distance between start and goal, and $\mu_{\mathrm{PHS}}(c, d)$ denotes the known Lebesgue measure of an $n$-dimensional PHS with transverse diameter $c$ and focal distance $d$~\cite{gammell2018informed}.

\looseness=-1
Equation~\eqref{eq:ris-volume} shows that, when $\Matrix{G}_{\lambda} \preceq \Glower$,
the matrix bound reduces the volume of the informed set through two independent mechanisms.
First, because $\hat{d} \geq \hat{d}_\lambda$, $d_{\mathrm{foci}}$ is larger under the matrix bound, producing a thinner PHS in the isotropic space and reducing the numerator.
Second, the denominator $\sqrt{\Determinant{\Glower}} = \sqrt{\prod_{i=1}^{n} \lambda_i(\Glower)}$ grows with the full eigenspectrum of $\Glower$, not just its minimum eigenvalue as under the scalar bound, yielding a strictly smaller informed set whenever $\Glower$ is anisotropic.
In high dimensions, even modest per-direction improvements accumulate multiplicatively through the determinant, producing informed sets that are orders of magnitude smaller than their scalar counterparts.

%% file: sections/experiments.tex
\vspace{-1mm}
\section{Experiments}
\label{sec:experiments}

\begin{figure}[t]
\centering
\includegraphics[width=0.68\linewidth]{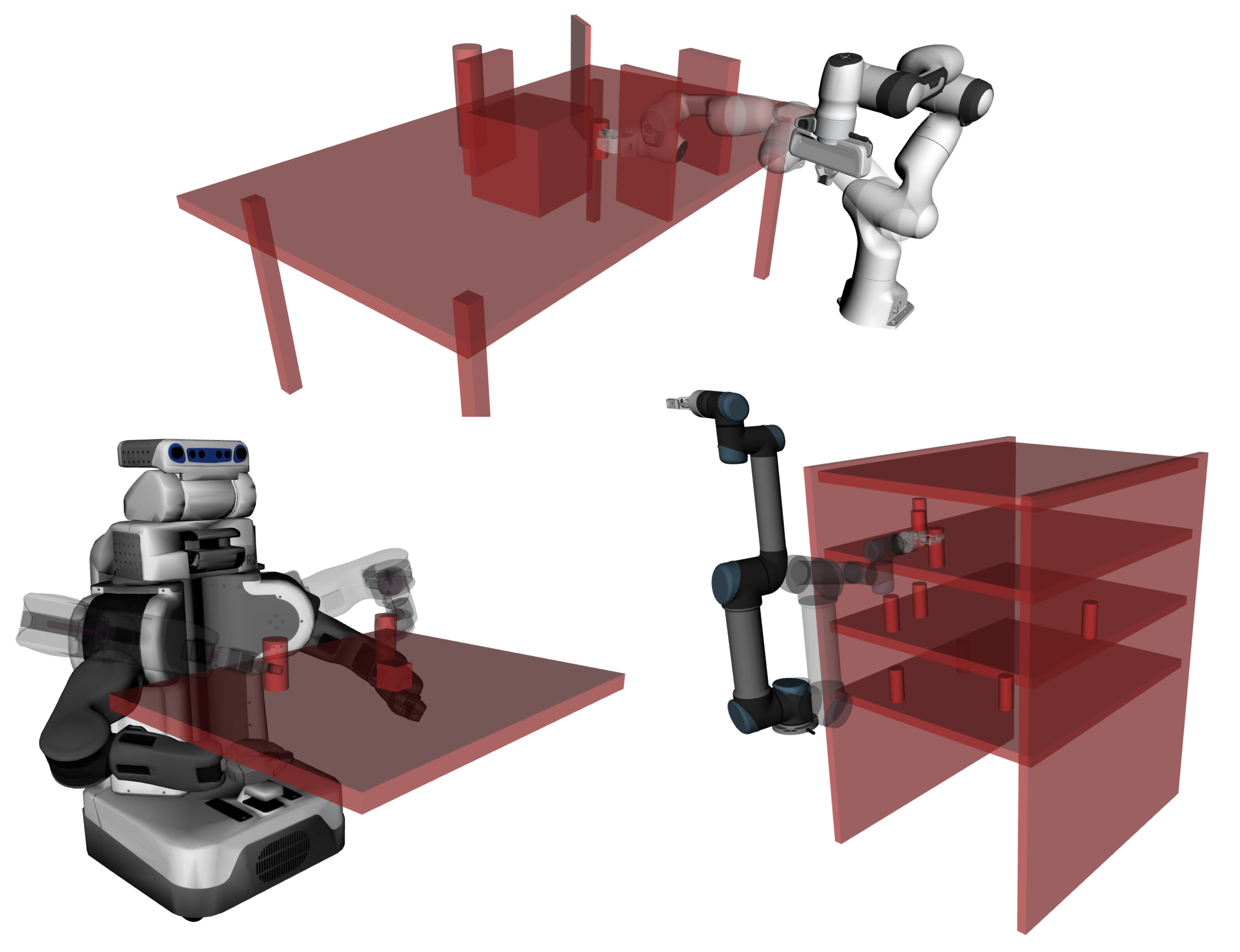}
\caption{
Planning scenes for the manipulation experiments in Section~\ref{sec:manipulation-problems}.
}
\label{fig:planning-scenes}
\end{figure}

\looseness=-1
We evaluate our matrix-valued heuristic on manipulation problems using three robots of increasing dimensionality: a 6-DoF UR5, a 7-DoF Franka, and a 14-DoF PR2 dual-arm system (Figure~\ref{fig:planning-scenes}).
The UR5 and Franka use the bookshelf and table pick environments of the MotionBenchMaker dataset~\cite{chamzas2021motionbenchmaker}, and the PR2 uses
a custom dual-arm scene.
Each provides 100 problems with different start and goal configurations.
For each robot, we consider a diagonal weighted metric, the kinetic energy metric derived from the robot's mass matrix, and a pullback metric induced by the manipulator Jacobian.
The weighted metric penalizes unnecessary joint motion by assigning a weight of 100 to joints with small start-to-goal displacement and 1 otherwise.
The pullback metric uses a regularization value of 0.1 to ensure positive definiteness near singularities.
All experiments minimize the Riemannian arc length~\eqref{eqn:arc-length},
using the midpoint approximation of~\cite{kyaw2026geometry}.

\looseness=-1
We compare five asymptotically optimal sampling-based planners, BIT*~\cite{gammell2020batch}, AIT* and EIT*~\cite{strub2022adaptively}, GRRT*~\cite{kyaw2024greedy}, and AORRTC~\cite{wilson2025aorrtc}, each under the Euclidean distance, the zero heuristic \mbox{($\hat{d} = 0$)}, the scalar eigenvalue bound, and our proposed Loewner lower bound.
For \mbox{GRRT*}, we set the greedy biasing ratio to $\epsilon = 0$, which disables greedy exploitation and
reduces the planner to a bidirectional informed RRT*.
We also include {GRRT*} with $\epsilon = 0.9$ and the matrix-valued heuristic to demonstrate that greedy informed sampling further accelerates convergence when combined with the Loewner bound.

\looseness=-1
All planners are implemented in C++ using the Open Motion Planning Library (OMPL)~\cite{sucan2012open} with VAMP's collision checking backend~\cite{thomason2024motions}.
The kinetic energy metric uses the mass matrix computed via Pinocchio~\cite{carpentier2019pinocchio}.
The Loewner lower bound $\Glower$ is precomputed for each robot--metric pair via Algorithm~\ref{alg:loewner-meet}, using an optimization-based approach (Algorithm~\ref{alg:opt-bound}) to select metric evaluation points rather than uniform random sampling (Appendix~\ref{sec:appendix-opt-bound}).
We find that this approach converges to a tighter bound more efficiently, particularly in high dimensions.
The weighted metric is constant, so its bound needs no search, while the kinetic energy and pullback bounds take 10.6/7.4\,s (UR5), 22.6/16.5\,s (Franka), and 156/328\,s (PR2).
For each problem, we run five trials with different pseudorandom seeds and a planning budget of 120\,s, and report the median solution cost over all 500 runs.

\begin{figure}[!t]
\centering
\includegraphics[width=0.82\linewidth]{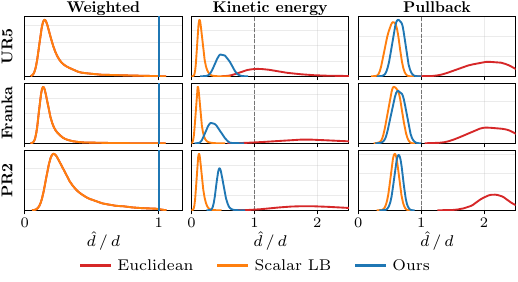}
\vspace{-2mm}
\caption{%
\looseness=-1
Distribution of the heuristic-to-geodesic-distance ratio $\hat{d} / d$ for 10,000 random configuration pairs per robot and metric; the dashed line marks the admissibility threshold at ratio 1.
The Euclidean heuristic crosses the threshold under the kinetic energy and pullback metrics and is therefore inadmissible on both.
The scalar and matrix bounds stay below the threshold, but the scalar bound is far looser.
Our matrix bound is exact under the weighted metric, 3--4 times tighter than the scalar bound under the kinetic energy metric, and close to it under the pullback metric, where regularization makes $\Glower$ nearly isotropic.
}
\label{fig:heuristic-ratio}
\end{figure}

\vspace{-2mm}
\subsection{Heuristic Quality}
\label{sec:exp-heuristic-quality}

\looseness=-1
This section evaluates the tightness and empirical admissibility of each heuristic.
For each robot and metric, we sample 10,000 random configuration pairs and compute the ratio $\hat{d}(\cdot) / d(\cdot)$,
where $d$ is obtained by solving the geodesic boundary value problem numerically, minimizing the discrete path energy with L-BFGS, refining the discretization from coarse to fine, and restarting the optimization from several perturbed initializations.
A ratio of 1 indicates a perfect heuristic; below 1 is admissible; above 1 is inadmissible.

\looseness=-1
Figure~\ref{fig:heuristic-ratio} shows the heuristic ratio distributions across all robots and metrics.
The Euclidean heuristic is inadmissible under the kinetic energy and pullback metrics, with median ratios exceeding 1 in all scenarios, confirming that it is unsuitable as a cost-to-go estimate under general Riemannian metrics.
The scalar bound is always admissible but highly conservative.
Under the constant weighted metric, the matrix bound recovers the geodesic distance exactly, trivially validating the approach since $\Glower = \Matrix{G}$.
The advantage is most pronounced under the kinetic energy metric, where the matrix bound is 3--4 times tighter than the scalar bound, capturing the directional structure of the inertia tensor.
Under the pullback metric, $\lambda$-regularization near kinematic singularities makes $\Glower$ nearly isotropic, limiting improvement over the scalar bound to approximately 10--15\%.
These per-metric differences directly explain the convergence patterns in Section~\ref{sec:manipulation-problems}.

\begin{figure*}[t]
\centering
\includegraphics[width=0.89\linewidth]{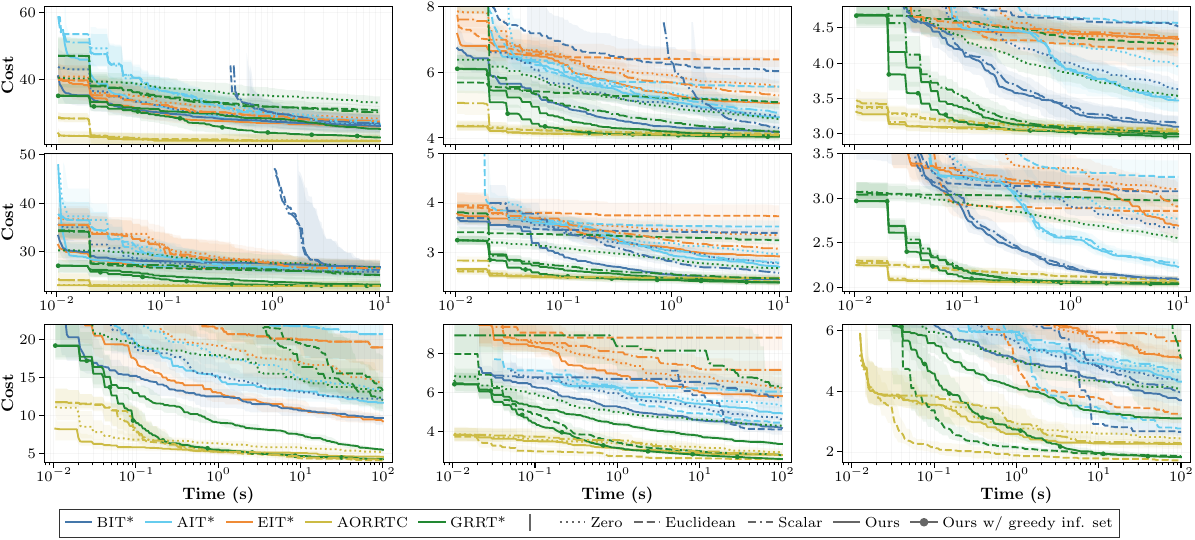}
\caption{%
Anytime planning performance across all nine robot--metric scenarios.
Rows correspond to robots (top to bottom: UR5, Franka, PR2) and columns to metrics (left to right: weighted Euclidean, kinetic energy, pullback).
Each panel plots median solution cost against planning time for all tested planners under the zero, Euclidean, scalar bound, and matrix lower bound heuristics; shaded regions denote nonparametric 99\% confidence intervals on the median.
}
\label{fig:convergence-plots}
\vspace{-4mm}
\end{figure*}

\vspace{-1mm}
\subsection{Manipulation Problems}
\label{sec:manipulation-problems}

\looseness=-1
Under the weighted and kinetic energy metrics, planners with the matrix bound consistently converge faster than their zero, Euclidean, and scalar counterparts, with the degree of improvement varying across planners and robot dimensionality (Figure~\ref{fig:convergence-plots}).
The improvement is most pronounced under the kinetic energy metric, consistent with the 3--4 times tighter heuristic ratios reported in Section~\ref{sec:exp-heuristic-quality}.
The scalar bound is admissible but too loose to focus the search.
As \eqref{eq:ris-volume} predicts, its informed set still covers most of the configuration space, so it performs no better than the zero heuristic on the UR5 and Franka, and worse on the 14-DoF PR2.

\looseness=-1
Under the pullback metric, the inadmissible Euclidean heuristic traps planners in local minima on the UR5 and Franka tasks, failing to converge within the time budget.
The PR2 is an exception: despite large overestimation, the Euclidean heuristic performs well, as the overestimate coincidentally guides search toward better solutions in this scenario.
The matrix bound offers only modest improvement over the scalar bound, as the regularization term near singularities makes $\Glower$ nearly isotropic, collapsing its directional advantage as predicted in Section~\ref{sec:exp-heuristic-quality}.

\looseness=-1
The fastest convergence is achieved with GRRT* under the matrix bound, across all robots and metrics.
Even when the bound becomes nearly isotropic and the informed set offers limited benefit, direct sampling in the isotropic Euclidean space still enables GRRT* to exploit its greedy biasing, maintaining its advantage over all baselines.
AORRTC with the matrix bound also converges faster than its other variants, though they all remain competitive across all metrics; heuristic quality affects only the informed sampling step in AORRTC's cost-augmented search.

%% file: sections/conclusion.tex
\vspace{-1mm}
\section{Conclusion}
\label{sec:conclusion}

\looseness=-1
This work presented a \emph{matrix-valued} admissible heuristic for informed sampling in motion planning under configuration-dependent Riemannian metrics.
The key insight is that a constant Loewner lower bound on the metric tensor preserves directional structure, yielding tighter informed sets than the scalar eigenvalue bound while remaining provably admissible.
Its Cholesky factorization reduces the Riemannian informed set to a standard prolate hyperspheroid in isotropic Euclidean space, enabling direct sampling with existing techniques.
Across nine manipulation scenarios spanning three robots and three metrics, our heuristic consistently accelerates convergence across multiple asymptotically optimal planners, with the largest gains under anisotropic metrics where directional structure is most informative.

\looseness=-1
Our development assumes a configuration space with a single global chart, so that coordinate differences are well defined.
Spaces admitting no such chart, for example, $\LieGroupSO{3}$, require an intrinsic map such as the Lie group logarithm, and the informed set is no longer an ellipsoid.
We leave the problem of recovering direct sampling in this setting as future work.
We are also interested in investigating tighter bounds for heavily regularized metrics, where the current approach degenerates toward the scalar bound.

%% file: sections/appendix.tex
\begin{algorithm}[!t]
\caption{Loewner lower bound ($q_{\mathrm{init}}$, $\tau$)}
\label{alg:opt-bound}
$\Matrix{L} \gets \mathrm{Cholesky}(\Matrix{G}(q_{\mathrm{init}}))$\;
\Repeat{$\lambda^* \geq 1 - \tau$}{
    $q^* \gets \ArgMin{q} \lmin(\Inv{\Matrix{L}}\, \Matrix{G}(q)\, \Matrix{L}^{-\T})$\;
    $\lambda^* \gets \lmin(\Inv{\Matrix{L}}\, \Matrix{G}(q^*)\, \Matrix{L}^{-\T})$\;
    \If{$\lambda^* < 1 - \tau$}{
        $\Matrix{L} \gets \mathrm{LoewnerMeet}(\Matrix{L},\, \Matrix{G}(q^*))$
        \tcp*{Alg.~\ref{alg:loewner-meet}}
    }
}
\Return{$\Glower = \Matrix{L}\Matrix{L}^\T$}
\end{algorithm}

\vspace{-2mm}
\appendices

\section{Optimization-Based Bound Computation}
\label{sec:appendix-opt-bound}

\looseness=-1
We initialize $\Glower$ with the metric evaluated at $q_{\mathrm{init}}$, set to the midpoint of the joint limits, and iteratively search for $q^*$ via multi-start gradient descent with Armijo backtracking (Algorithm~\ref{alg:opt-bound}).
At each iteration, $q^*$ is the configuration with the largest deviation between $\Matrix{G}(q)$ and the current bound; we then tighten $\Glower$ via Algorithm~\ref{alg:loewner-meet}.
The process terminates when $\lambda^* \geq 1 - \tau$ for all found optima, which verifies $(1-\tau)\Glower \preceq \Matrix{G}(q)$ at the computed local optima.
We use $\tau = 10^{-6}$ in all experiments.

%% file: main.bbl
\begin{thebibliography}{10}
\providecommand{\url}[1]{#1}
\csname url@samestyle\endcsname
\providecommand{\newblock}{\relax}
\providecommand{\bibinfo}[2]{#2}
\providecommand{\BIBentrySTDinterwordspacing}{\spaceskip=0pt\relax}
\providecommand{\BIBentryALTinterwordstretchfactor}{4}
\providecommand{\BIBentryALTinterwordspacing}{\spaceskip=\fontdimen2\font plus
\BIBentryALTinterwordstretchfactor\fontdimen3\font minus
  \fontdimen4\font\relax}
\providecommand{\BIBforeignlanguage}[2]{{%
\expandafter\ifx\csname l@#1\endcsname\relax
\typeout{** WARNING: IEEEtran.bst: No hyphenation pattern has been}%
\typeout{** loaded for the language `#1'. Using the pattern for}%
\typeout{** the default language instead.}%
\else
\language=\csname l@#1\endcsname
\fi
#2}}
\providecommand{\BIBdecl}{\relax}
\BIBdecl

\bibitem{lavalle2001randomized}
S.~M. LaValle and J.~J. Kuffner, Jr., ``Randomized kinodynamic planning,''
  \emph{Int. J. Robot. Res.}, vol.~20, no.~5, pp. 378--400, 2001.

\bibitem{karaman2011sampling}
S.~Karaman and E.~Frazzoli, ``Sampling-based algorithms for optimal motion
  planning,'' \emph{Int. J. Robot. Res.}, vol.~30, no.~7, pp. 846--894, 2011.

\bibitem{gammell2020batch}
J.~D. Gammell, T.~D. Barfoot, and S.~S. Srinivasa, ``{Batch Informed Trees
  (BIT*)}: Informed asymptotically optimal anytime search,'' \emph{Int. J.
  Robot. Res.}, vol.~39, no.~5, pp. 543--567, 2020.

\bibitem{gammell2018informed}
------, ``Informed sampling for asymptotically optimal path planning,''
  \emph{{IEEE} Trans. Robot.}, vol.~34, no.~4, pp. 966--984, 2018.

\bibitem{bullo2019geometric}
F.~Bullo and A.~D. Lewis, \emph{Geometric Control of Mechanical Systems:
  Modeling, Analysis, and Design for Simple Mechanical Control Systems}, ser.
  Texts in Applied Mathematics.\hskip 1em plus 0.5em minus 0.4em\relax
  Springer, 2005, vol.~49.

\bibitem{jaquier2022riemannian}
N.~Jaquier and T.~Asfour, ``{Riemannian} geometry as a unifying theory for
  robot motion learning and control,'' in \emph{Proc. Int. Symp. Robot. Res.
  {(ISRR)}}, 2023, pp. 395--403.

\bibitem{li2024riemannian}
Y.~Li, J.~Qiu, and S.~Calinon, ``A {Riemannian} take on distance fields and
  geodesic flows in robotics,'' \emph{Int. J. Robot. Res.}, vol.~45, no.~13,
  pp. 2138--2167, 2026.

\bibitem{kyaw2026geometry}
P.~T. Kyaw and J.~Kelly, ``Geometry-aware sampling-based motion planning on
  {Riemannian} manifolds,'' in \emph{Proc. World Symp. Algorithmic Found.
  Robot. {(WAFR)}}, 2026.

\bibitem{jaquier2021geometry}
N.~Jaquier, L.~Rozo, D.~G. Caldwell, and S.~Calinon, ``Geometry-aware
  manipulability learning, tracking, and transfer,'' \emph{Int. J. Robot.
  Res.}, vol.~40, no. 2--3, pp. 624--650, 2021.

\bibitem{saveriano2023learning}
M.~Saveriano, F.~J. Abu-Dakka, and V.~Kyrki, ``Learning stable robotic skills
  on {Riemannian} manifolds,'' \emph{Robot. Auton. Syst.}, vol. 169, p. 104510,
  2023.

\bibitem{maric2021riemannian}
F.~Mari{\'c}, L.~Petrovi{\'c}, M.~Guberina, J.~Kelly, and I.~Petrovi{\'c}, ``A
  {Riemannian} metric for geometry-aware singularity avoidance by articulated
  robots,'' \emph{Robot. Auton. Syst.}, vol. 145, p. 103865, 2021.

\bibitem{peyre2009geodesic}
G.~Peyr{\'e} and L.~D. Cohen, ``Geodesic methods for shape and surface
  processing,'' in \emph{Adv. Comput. Vis. Med. Image Process.}, 2009, pp.
  29--56.

\bibitem{strub2022adaptively}
M.~P. Strub and J.~D. Gammell, ``{Adaptively Informed Trees (AIT*)} and {Effort
  Informed Trees ({EIT*})}: Asymmetric bidirectional sampling-based path
  planning,'' \emph{Int. J. Robot. Res.}, vol.~41, no.~4, pp. 390--417, 2022.

\bibitem{kyaw2024greedy}
P.~T. Kyaw, A.~V. Le, R.~E. Mohan, and J.~Kelly, ``Greedy heuristics for
  sampling-based motion planning in high-dimensional state spaces,''
  \emph{Auton. Robots}, vol.~50, no.~4, p.~44, 2026.

\bibitem{sehn2024off}
J.~Sehn, T.~D. Barfoot, and J.~Collier, ``Off the beaten track: Laterally
  weighted motion planning for local obstacle avoidance,'' \emph{IEEE Trans.
  Field Robot.}, vol.~1, pp. 249--275, 2024.

\bibitem{yi2018generalizing}
D.~Yi, R.~Thakker, C.~Gulino, O.~Salzman, and S.~Srinivasa, ``Generalizing
  informed sampling for asymptotically-optimal sampling-based kinodynamic
  planning via {M}arkov chain {M}onte {C}arlo,'' in \emph{Proc. {IEEE} Int.
  Conf. Robot. Autom. {(ICRA)}}, 2018, pp. 7063--7070.

\bibitem{kunz2016hierarchical}
T.~Kunz, A.~Thomaz, and H.~Christensen, ``Hierarchical rejection sampling for
  informed kinodynamic planning in high-dimensional spaces,'' in \emph{Proc.
  {IEEE} Int. Conf. Robot. Autom. {(ICRA)}}, 2016, pp. 89--96.

\bibitem{gammell2021asymptotically}
J.~D. Gammell and M.~P. Strub, ``Asymptotically optimal sampling-based motion
  planning methods,'' \emph{Annu. Rev. Control Robot. Auton. Syst.}, vol.~4,
  no.~1, pp. 295--318, 2021.

\bibitem{paden2017verification}
B.~Paden, V.~Varricchio, and E.~Frazzoli, ``Verification and synthesis of
  admissible heuristics for kinodynamic motion planning,'' \emph{{IEEE} Robot.
  Autom. Lett.}, vol.~2, no.~2, pp. 648--655, 2017.

\bibitem{przybylski2024asymptotically}
M.~Przybylski, ``Asymptotically optimal {A*} for kinodynamic planning,''
  \emph{{IEEE} Robot. Autom. Lett.}, vol.~9, no.~5, pp. 4353--4360, 2024.

\bibitem{sakcak2019admissible}
B.~Sakcak, L.~Bascetta, G.~Ferretti, and M.~Prandini, ``An admissible heuristic
  to improve convergence in kinodynamic planners using motion primitives,''
  \emph{IEEE Control Syst. Lett.}, vol.~4, no.~1, pp. 175--180, 2020.

\bibitem{liu2018search}
S.~Liu, K.~Mohta, N.~Atanasov, and V.~Kumar, ``Search-based motion planning for
  aggressive flight in {SE(3)},'' \emph{{IEEE} Robot. Autom. Lett.}, vol.~3,
  no.~3, pp. 2439--2446, 2018.

\bibitem{surazhsky2005fast}
V.~Surazhsky, T.~Surazhsky, D.~Kirsanov, S.~J. Gortler, and H.~Hoppe, ``Fast
  exact and approximate geodesics on meshes,'' \emph{ACM Trans. Graph.},
  vol.~24, no.~3, pp. 553--560, 2005.

\bibitem{peyre2006landmark}
G.~Peyr{\'e} and L.~D. Cohen, ``Landmark-based geodesic computation for
  heuristically driven path planning,'' in \emph{Proc. {IEEE} Conf. Comput.
  Vis. Pattern Recognit. {(CVPR)}}, vol.~2, 2006, pp. 2229--2236.

\bibitem{paden2017landmark}
B.~Paden, Y.~Nager, and E.~Frazzoli, ``Landmark guided probabilistic roadmap
  queries,'' in \emph{Proc. {IEEE/RSJ} Int. Conf. Intell. Robots Syst.
  {(IROS)}}, 2017, pp. 4828--4834.

\bibitem{mirebeau2019riemannian}
J.-M. Mirebeau, ``Riemannian fast-marching on {Cartesian} grids, using
  {Voronoi's} first reduction of quadratic forms,'' \emph{SIAM J. Numer.
  Anal.}, vol.~57, no.~6, pp. 2608--2655, 2019.

\bibitem{cohn2025non}
T.~Cohn, M.~Petersen, M.~Simchowitz, and R.~Tedrake, ``Non-{Euclidean} motion
  planning with graphs of geodesically convex sets,'' \emph{Int. J. Robot.
  Res.}, vol.~44, no. 10--11, pp. 1840--1862, 2025.

\bibitem{lukyanenko2023probabilistic}
A.~Lukyanenko and D.~Soudbakhsh, ``Probabilistic motion planning for
  non-{Euclidean} and multi-vehicle problems,'' \emph{Robot. Auton. Syst.},
  vol. 168, p. 104487, 2023.

\bibitem{boyd1994linear}
S.~Boyd, L.~El~Ghaoui, E.~Feron, and V.~Balakrishnan, \emph{Linear Matrix
  Inequalities in System and Control Theory}.\hskip 1em plus 0.5em minus
  0.4em\relax SIAM, 1994.

\bibitem{bopardikar2016robust}
S.~D. Bopardikar, B.~Englot, A.~Speranzon, and J.~van~den Berg, ``Robust belief
  space planning under intermittent sensing via a maximum eigenvalue-based
  bound,'' \emph{Int. J. Robot. Res.}, vol.~35, no.~13, pp. 1609--1626, 2016.

\bibitem{shan2017belief}
T.~Shan and B.~Englot, ``Belief roadmap search: Advances in optimal and
  efficient planning under uncertainty,'' in \emph{Proc. {IEEE/RSJ} Int. Conf.
  Intell. Robots Syst. {(IROS)}}, 2017, pp. 5318--5325.

\bibitem{singh2017robust}
S.~Singh, A.~Majumdar, J.-J. Slotine, and M.~Pavone, ``Robust online motion
  planning via contraction theory and convex optimization,'' in \emph{Proc.
  {IEEE} Int. Conf. Robot. Autom. {(ICRA)}}, 2017, pp. 5883--5890.

\bibitem{kalakrishnan2011stomp}
M.~Kalakrishnan, S.~Chitta, E.~Theodorou, P.~Pastor, and S.~Schaal, ``{STOMP}:
  Stochastic trajectory optimization for motion planning,'' in \emph{Proc.
  {IEEE} Int. Conf. Robot. Autom. {(ICRA)}}, 2011, pp. 4569--4574.

\bibitem{zucker2013chomp}
M.~Zucker, N.~Ratliff, A.~D. Dragan, M.~Pivtoraiko, M.~Klingensmith, C.~M.
  Dellin, J.~A. Bagnell, and S.~S. Srinivasa, ``{CHOMP}: Covariant
  {H}amiltonian optimization for motion planning,'' \emph{Int. J. Robot. Res.},
  vol.~32, no. 9--10, pp. 1164--1193, 2013.

\bibitem{deits2015computing}
R.~Deits and R.~Tedrake, ``Computing large convex regions of obstacle-free
  space through semidefinite programming,'' in \emph{Proc. Workshop Algorithmic
  Found. Robot. {(WAFR)}}, 2015, pp. 109--124.

\bibitem{holmes2024semidefinite}
C.~Holmes, F.~D{\"u}mbgen, and T.~D. Barfoot, ``On semidefinite relaxations for
  matrix-weighted state-estimation problems in robotics,'' \emph{{IEEE} Trans.
  Robot.}, vol.~40, pp. 4805--4824, 2024.

\bibitem{lee2018introduction}
J.~M. Lee, \emph{Introduction to {Riemannian} Manifolds}, 2nd~ed., ser.
  Graduate Texts in Mathematics.\hskip 1em plus 0.5em minus 0.4em\relax
  Springer, 2018, vol. 176.

\bibitem{bhatia2007positive}
R.~Bhatia, \emph{Positive Definite Matrices}.\hskip 1em plus 0.5em minus
  0.4em\relax Princeton University Press, 2007.

\bibitem{kadison1951order}
R.~V. Kadison, ``Order properties of bounded self-adjoint operators,''
  \emph{Proc. Amer. Math. Soc.}, vol.~2, no.~3, pp. 505--510, 1951.

\bibitem{stott2016maximal}
N.~Stott, ``Maximal lower bounds in the {L{\"o}wner} order,'' \emph{Proc. Amer.
  Math. Soc.}, 2017, doi: 10.1090/proc/13785.

\bibitem{chamzas2021motionbenchmaker}
C.~Chamzas, C.~Quintero-Pe{\~n}a, Z.~Kingston, A.~Orthey, D.~Rakita,
  M.~Gleicher, M.~Toussaint, and L.~E. Kavraki, ``{MotionBenchMaker}: A tool to
  generate and benchmark motion planning datasets,'' \emph{{IEEE} Robot. Autom.
  Lett.}, vol.~7, no.~2, pp. 882--889, 2022.

\bibitem{wilson2025aorrtc}
T.~S. Wilson, W.~Thomason, Z.~Kingston, and J.~D. Gammell, ``{AORRTC}:
  Almost-surely asymptotically optimal planning with {RRT-Connect},''
  \emph{{IEEE} Robot. Autom. Lett.}, vol.~10, no.~12, pp. 13\,375--13\,382,
  2025.

\bibitem{sucan2012open}
I.~A. {\c{S}}ucan, M.~Moll, and L.~E. Kavraki, ``The {Open Motion Planning
  Library},'' \emph{{IEEE} Robot. Autom. Mag.}, vol.~19, no.~4, pp. 72--82,
  2012.

\bibitem{thomason2024motions}
W.~Thomason, Z.~Kingston, and L.~E. Kavraki, ``Motions in microseconds via
  vectorized sampling-based planning,'' in \emph{Proc. {IEEE} Int. Conf. Robot.
  Autom. {(ICRA)}}, 2024, pp. 8749--8756.

\bibitem{carpentier2019pinocchio}
J.~Carpentier, G.~Saurel, G.~Buondonno, J.~Mirabel, F.~Lamiraux, O.~Stasse, and
  N.~Mansard, ``The {Pinocchio C++} library: A fast and flexible implementation
  of rigid body dynamics algorithms and their analytical derivatives,'' in
  \emph{Proc. {IEEE/SICE} Int. Symp. Syst. Integr. {(SII)}}, 2019, pp.
  614--619.

\end{thebibliography}
